\documentclass{article}

\usepackage{microtype}
\usepackage{graphicx}
\usepackage{booktabs} 
\usepackage{subcaption}

\usepackage{hyperref}

\usepackage[accepted]{icml2025}

\usepackage{amsmath}
\usepackage{amssymb}
\usepackage{mathtools}
\usepackage{amsthm}
\usepackage{bbm}
\usepackage{dsfont}
\usepackage[capitalize,noabbrev]{cleveref}
\theoremstyle{plain}
\newtheorem{theorem}{Theorem}[section]

\newtheorem{lemma}[theorem]{Lemma}

\theoremstyle{definition}

\theoremstyle{remark}

\usepackage[textsize=tiny]{todonotes}
\usepackage[justification=centering]{caption}

\icmltitlerunning{Doubly Robust Estimation of Causal Effect on CVR with Targeted Regularization}

\begin{document}

\twocolumn[
\icmltitle{Doubly Robust Estimation of Causal Effect on CVR with Targeted Regularization}




\begin{icmlauthorlist}
\icmlauthor{Jiayi Dan}{xxx}
\icmlauthor{Bo Li}{xxx}
\icmlauthor{Lu Deng}{yyy}
\icmlauthor{Yong Wang}{yyy}
\end{icmlauthorlist}

\icmlaffiliation{xxx}{Tsinghua University, Beijing, China}
\icmlaffiliation{yyy}{Tencent Inc., Shenzhen, China}

\icmlcorrespondingauthor{Jiayi Dan, Bo Li}{danjy24@tsinghua.org.cn, libo@sem.tsinghua.edu.cn}

\icmlkeywords{Machine Learning, ICML}

\vskip 0.3in
]



\printAffiliationsAndNotice{} 

\begin{abstract}

Post-click conversion rate (CVR) is a key metric in various scenarios including e-commerce and advertising, reflecting the efficiency and user experience in the second stage of the conversion process. Estimating the causal effect on CVR is therefore of great practical importance. However, directly applying existing causal inference methods to clicked samples introduces sample selection bias and increased variance due to the exclusion of non-click data. Recent studies on CVR prediction introduce “ideal loss”, which optimizes model parameters using an unbiased estimate of the loss over the full sample. Nevertheless, there is no guarantee that unbiasedness of the loss implies unbiasedness of the final estimator. 

We revisit this challenge from the perspective of semiparametric theory. Specifically, we develop a new doubly robust causal effect estimator for chain-structured outcomes such as CVR, and derive its theoretical properties in detail. It achieves a faster convergence rate compared to nuisance parameters estimation and is therefore more robust when using flexible nonparametric estimators, including neural networks. Based on these theoretical findings, we further design a framework based on targeted regularization to improve numerical stability and practical applicability.
 
Extensive experiments on synthetic and real-world data demonstrate the effectiveness and robustness of our method. In addition, we find that naively combining loss debiasing with standard causal estimators underperforms our method, highlighting the necessity of developing the new estimator tailored to this CVR-style objective with solid theoretical guarantees. 

\end{abstract}

\section{Introduction}
Causal effect estimation plays a critical role in enterprise decision-making.
In e-commerce and ride-hailing scenarios, platforms distribute coupons to stimulate user transactions.
Existing research has extensively studied the causal effect of strategies on conversion probability,
yet the stage-wise mechanism behind the conversion process has received relatively less attention.
For example, high-value coupons with strict usage requirements can substantially increase users' click probability while simultaneously reducing post-click conversion probability, which indicate potential deterioration in user experience.
Industry practitioners have recognized this issue and introduced \emph{post-click conversion rate (CVR)}
as a core metric describing second-stage conversion efficiency.
Let $Y_1$ and $Y_2$ denote whether a user clicks and whether conversion occurs (with $Y_2=1$ implying $Y_1=1$);
CVR is defined as $\mathbb{P}(Y_2=1\mid Y_1=1)$.
A healthy policy should improve overall conversion probabilities while ensuring stage-wise efficiency does not degrade.

Under the potential outcomes framework, let $A$ and $X$ denote treatment and covariates, we define the conditional and average potential outcomes of CVR as:
\begin{equation*}
\begin{aligned}
\psi_a(x)
&= \mathbb{P}\!\left(Y_2(a)=1 \mid Y_1(a)=1,\, X=x\right),
\\[6pt]
\psi_a
&= \mathbb{E}\!\left[\mathbb{P}\!\left(Y_2(a)=1 \mid Y_1(a)=1,\, X\right)\right],
\end{aligned}
\end{equation*}
which is defined on the entire population under the strategy.  

From a mechanism decomposition perspective, mediation analysis is another classic framework, aiming to decompose the total effect into direct and indirect effects. A key step is estimating cross-world expectations $\mathbb{E}[Y(a,M(a^*))]$ where $M$ denotes the mediator. We cannot simply treat $Y_1$ as a mediator since $Y_1$ and $Y_2$ are always defined under the same treatment value, and mediation analysis degenerates into standard causal inference without cross-world quantities, since $\mathbb{E}[Y(a,M(a))]\equiv \mathbb{E}[Y(a)]$. To circumvent cross-world estimation, \cite{robins1992identifiability} proposed the \textit{controlled direct effect} (CDE), which is characterized by the estimand $\mathbb{E}[\mathbb{E}(Y(a,m)\mid X)]$.
We cannot adopt related methods for two main reasons: (1) the identifiability assumptions $M \perp\!\!\!\perp Y(a,m)\mid A,X$ rarely hold as the dependency between click and conversion cannot be blocked, (2) the key nuisance parameter $\mathbb{E}(Y_2\mid Y_1=1,X=x,A=a)$ cannot be estimated directly. We cannot regress $Y_2$ on $Y_1$, $X$, and $A$ as in mediation analysis since $Y_2$ is strictly dependent on $Y_1$ ($Y_1 = 0 \rightarrow Y_2 = 0$), and selecting samples with $Y_1 = 1$ would introduce selection bias. 

The preceding discussion leads to the first challenge: selection bias. Most strategies, such as coupon allocation, are implemented prior to any click behavior, and our focus is on the full population. Therefore, directly treating $Y_2$ as the outcome and applying existing causal estimators to samples with $Y_1=1$ would introduce bias due to the distribution shift induced by click selection.
\cite{huang2024entire} demonstrated that estimating treatment effect on CVR only with clicked samples will lead to incorrect conclusions. However, its proposed method  applies only to RCT data and lacks sound theoretical properties. Another line of work introduces the notion of \emph{ideal loss} to correct selection bias 
and develops related CVR prediction models \cite{wang2019doubly,zhang2020large,guo2021enhanced,dai2022generalized,li2022stabledr,li2022tdr}. These methods construct unbiased estimators of the loss over full samples, 
yet no theory directly guarantees that unbiased loss implies unbiased estimation of the target estimand. 

To address these limitations, we propose a new estimator for CVR causal effect with both practical applicability and theoretical guarantees. It directly targets the final estimand instead of relying on loss debiasing. Specifically, we derive the influence function and the von Mises expansion of the target estimand based on semiparametric theory. This inspires us to construct a doubly robust estimator, which remains consistent even if one of the nuisance estimators is inconsistent, and achieves desirable asymptotic properties. Leveraging the above theoretical results, we extend the idea of targeted regularization to develop a more stable estimation framework.

Our contributions can be summarized as follows: 

(1) We revisit the problem from a semiparametric perspective. Specifically, we develop a new doubly robust estimator tailored to chain-structured outcomes, especially CVR, which is a nontrivial constructive problem, and further show its desirable theoretical properties based on the von Mises expansion.


(2) Building on the above findings, we develop a practical estimation framework based on targeted-regularization to 
enhance applicability and empirical performance.

(3) We conduct extensive experiments on synthetic, semi-synthetic, and real-world data, demonstrating the strong performance and robustness of our proposed method.

\section{Related Works}

\subsection{Causal Inference based on Deep Learning}
Deep learning has become a key tool for causal effect estimation due to its flexibility. Many studies focus on mitigating confounding bias through balanced representation learning. \cite{johansson2016learning} and \cite{shalit2017estimating} first established a generalization bound incorporating empirical risk and Integral Probability Metric (IPM) distance. Building on this, several works introduce weighting strategies\cite{johansson2018learning,hassanpour2019counterfactual,assaad2021counterfactual}. \cite{wang2022generalization} and  \cite{kazemi2024adversarially} extend this paradigm to continuous treatment setting.
Meanwhile, another classical neural-based causal estimation paradigm emphasizes propensity-score estimation \cite{shi2019adapting,nie2021vcnet}. In particular, \cite{nie2021vcnet} propose a spline-based variable-coefficient neural network for continuous treatment settings, which prevents the treatment signal from being overshadowed by other high-dimensional covariates. As industry applications continue to grow, some recent studies adopt more flexible network structures\cite{ke2021addressing,zhong2022descn,liu2023explicit}.
Beyond standard feedforward networks, specialized architectures have also been explored, including VAE \cite{louizos2017causal}, GAN \cite{yoon2018ganite} and diffusion models\cite{sanchez2022diffusion}. However, directly applying the above methods to causal effect estimation on CVR can only be conducted on the clicked subset, which may lead to selection bias and data sparsity issues.

\subsection{Doubly Robust Estimator and Targeted Regularization}
To address the estimation bias in the plug-in estimator, the DR learner—first proposed by \cite{van2005statistical} and further developed by \cite{chernozhukov2017double} and \cite{chernozhukov2018double}—achieves doubly robustness by incorporating a bias correction term. Their theoretical work demonstrated that these doubly robust estimators attain $\sqrt{n}$-convergence rates under appropriate conditions. 
Following this framework, \cite{nie2021quasi} introduced the R-Learner, a general class of two-step algorithms for estimating treatment effects. \cite{kennedy2024semiparametric} provides a comprehensive review of doubly robust estimation, with emphasis on its asymptotic properties and minimax-style efficiency bounds. To overcome the finite-sample instability of the standard doubly robust estimator based on one-step correction\cite{kennedy2024semiparametric},\cite{shi2019adapting} introduced targeted regularization inspired by TMLE \cite{van2011targeted}, which corrects estimation bias through an additional perturbation parameter, \cite{nie2021vcnet} further extends this practical framework to continuous-treatment setting. The above studies provide a solid theoretical foundation for establishing our new framework tailored to CVR. However, our target estimand and nuisance parameters differ fundamentally from those in standard causal-inference problems, requiring corresponding extensions and adaptations.

\subsection{Debiased CVR Prediction}
Sample-selection bias is a major challenge in CVR prediction, motivating many researchers to employ techniques from causal inference to address this problem. \cite{wang2019doubly} first propose the “ideal loss” concept, seeking to derive an unbiased estimate of the training loss over the entire sample. Building on this idea, follow-up studies refine the estimation method to improve  accuracy and reduce the variance \cite{zhang2020large,guo2021enhanced,dai2022generalized,li2022stabledr,li2022tdr}. However, this series of studies focus on CVR prediction instead of causal effect estimation. Moreover, their starting point is to estimate the loss over the entire sample, not to target the final estimand, while no theoretical results indicate that an unbiased loss necessarily leads to an unbiased final estimate.
For causal-effect estimation on CVR,  \cite{huang2024entire} demonstrate through real-world data analysis that restricting to clicked samples may yield misleading conclusions,
and propose a multitask model for estimation of CTR and CVR treatment effects. This work further corroborates the importance of our problem, but the proposed method is only applicable to randomized-experiment data. In addition, the approach mainly focuses on network-structure design and can be regarded as a plug-in estimator, which lacks desirable theoretical properties including double robustness and consistency.

\textbf{Research gap:} To the best of our knowledge, few studies have directly targeted the final estimand for CVR causal effect estimation while providing both sound theoretical guarantees and strong practical utility.

\section{Problem Formulation and Notations}

$\textbf{Notation}$ We use $\mathbb{E}$ to denote expectation, $\mathbb{P} \in \mathcal{P}$ to denote true probability measure and we write $\mathbb{P}(f) = \int f(\mathbf{z}) d\mathbb{P}(\mathbf{z})$, where $\mathcal{P}$ is a set of possible probability distributions.
Similarly, we use $\mathbb{P}_n$ to denote the empirical measure and we write $\mathbb{P}_n(f) = \int f(\mathbf{z}) d\mathbb{P}_n(z)$.  We denote convergence in distribution by $\overset{d}{\rightarrow}$ and convergence in probability by $\overset{p}{\rightarrow}$. 
$X_n = O_{\mathbb{P}}(r_n)$ means $X_n / r_n$ is bounded in probability and $X_n = o_{\mathbb{P}}(r_n)$ means $X_n / r_n \overset{p}{\rightarrow} 0$.$\tau$ denotes Rademacher random variables. We denote Rademacher complexity of a function class $\mathcal{F}: \mathcal{X} \to \mathbb{R}$ as $\text{Rad}_n(\mathcal{F}) = \mathbb{E}(\sup_{f\in\mathcal{F}}|\frac{1}{n}\sum_{i=1}^n \tau_i f(X_i)|)$. Given two functions $f_1, f_2: \mathcal{X} \to \mathbb{R}$, we define $\|f_1-f_2\|_{\infty} = \sup_{x\in\mathcal{X}}|f_1(x)-f_2(x)|$ and $\|f_1-f_2\|_{L^2} = (\int_{x\in\mathcal{X}}(f_1(x)-f_2(x))^2 dx)^{1/2}$. For a function class $\mathcal{F}$, we define $\|\mathcal{F}\|_{\infty} = \sup_{f\in\mathcal{F}}\|f\|_{\infty}$. 
$a_n \asymp b_n$ denotes that both $a_n/b_n$ and $b_n/a_n$ are bounded. $\delta(\cdot)$ denotes the Dirac measure, which is used only as a notational device in theoretical analysis.

Let $(\mathbf{z}_1,\ldots,\mathbf{z}_n)$ be i.i.d. samples drawn from $\mathbb{P}$, and denote $\mathbf{z}_i = (\mathbf{x}_i, a_i, y_{1i}, y_{2i})$,
where $\mathbf{x}_i \in \mathcal{X} \subset \mathbb{R}^d$ is the covariate vector, and $a_i \in \mathcal{A} \subset \mathbb{R}$ is the treatment.
Here $y_{1i} \in \{0,1\}$ indicates whether a click occurs, and $y_{2i}$ indicates whether a conversion occurs
, where $y_{2i}=1$ implies $y_{1i}=1$.
We use uppercase $\mathbf{Z},X,A,Y_1,Y_2$ to denote random variables corresponding to the observed lower-case ones. To ensure general applicability, we denote the treatment variable $A$ as continuous, which naturally extends to the discrete treatment case \cite{nie2021vcnet}.

Our target estimand is defined as:
$
\psi_a(\mathbb{P}) = \mathbb{E}\big[\mathbb{P}(Y_2(a)=1 \mid Y_1(a)=1,\mathbf{X})\big].
$,
which is a functional regarding the data-generating distribution $\mathbb{P}$. For standard causal inference problems, the estimand is typically $\psi_a=\mathbb{E}[\mathbb{E}(Y(a)\mid\mathbf{X})]$.
In our setting, $A$ is not randomly assigned, and both $Y_1$ and $Y_2$ depend on $A$. A naive strategy is to restrict estimation to samples with $Y_1=1$ and estimate $\mathbb{E}[\mathbb{P}(Y_2(a)=1\mid\mathbf{X})]$ with existing methods. 
However, this introduces bias since $Y_1$ is non-random. 

We define $\mu_1(\mathbf{x},a)=\mathbb{E}(Y_1\mid\mathbf{X}=\mathbf{x},A=a$,  $\mu_2(\mathbf{x},a)=\mathbb{E}(Y_2\mid\mathbf{X}=\mathbf{x},A=a)$, 
and let $\pi(a\mid \mathbf{x})$ denote the treatment density given $\mathbf{X}=\mathbf{x}$ (propensity score).
Our estimation procedure will involve three nuisance parameters $(\mu_1,\mu_2,\pi)$.
For notational simplicity in later derivations, we write $\mu_1,\mu_2,\pi$ and $\hat{\mu}_1,\hat{\mu}_2,\hat{\pi}$ to refer to $\mu_1(\mathbf{x},a),\mu_2(\mathbf{x},a),\pi(a\mid\mathbf{x})$ and $\hat{\mu}_1(\mathbf{x},a),\hat{\mu}_2(\mathbf{x},a),\hat{\pi}(a\mid\mathbf{x})$.

Our analysis is based on the following assumptions:
\begin{enumerate}
    \item[(i)] \textbf{Overlap.} There exists a constant $c>0$ such that for all $\mathbf{x} \in \mathcal{X}$ and $a \in \mathcal{A}$, $\pi(a\mid \mathbf{x}) \ge c$;
    \item[(ii)] \textbf{Unconfoundedness.} Covariates $\mathbf{X}$ block all back-door paths between treatment $A$ and outcomes $Y_1,Y_2$;
    \item[(iii)] There exists a constant $c>0$ such that 
    $\mu_1(\mathbf{x},a)\ge c$.
\end{enumerate}

Assumptions (i) and (ii) are standard in causal inference.
Assumption (iii) is reasonable in practice since the population is finite and click probability usually has a positive lower bound, samples with $\mu_1=0$ have undefined CVR and are outside our scope.

Under the assumptions above,

\begin{align*}
\psi_a(\mathbb{P})
&= \mathbb{E}\!\left[\frac{\mathbb{P}(Y_2(a)=1,Y_1(a)=1\mid\mathbf{X}=)}{\mathbb{P}(Y_1(a)=1\mid\mathbf{X})}\right] \\
&= \mathbb{E}\!\left[\frac{\mathbb{E}(Y_2(a)\mid\mathbf{X})}{\mathbb{E}(Y_1(a)\mid\mathbf{X})}\right]
= \mathbb{E}\!\left[\frac{\mathbb{E}(Y_2\mid\mathbf{X},A=a)}{\mathbb{E}(Y_1\mid\mathbf{X},A=a)}\right].
\end{align*}

We fully exploit the dependence between $Y_1$ and $Y_2$, which allows all nuisance parameters to be estimated using the full sample.

\section{Doubly Robust Estimation}
\subsection{Plug-in Estimator and Nuisance Parameter Estimation}
We begin with a naive plug-in estimator to illustrate the estimation of nuisance components and the limitations of this approach. 
To leverage the full dataset, the simplest strategy is to predict $\mu_1(\mathbf{x},a)$ and $\mu_2(\mathbf{x},a)$ using covariates and the treatment. The final estimate is obtained by taking their ratio:
\[
\hat{\psi}^{\text{plug-in}} = \frac{1}{n} \sum_{i=1}^{n} \frac{\hat{\mu}_2(\mathbf{x}_i,a)}{\hat{\mu}_1(\mathbf{x}_i,a)}.
\]

This is similar to the approach for CVR prediction used in ESMM \cite{ma2018entire}. 
For causal estimation, following \cite{shi2019adapting}, we extract representations related to the treatment $A$ by predicting $\pi(a\mid\mathbf{x})$ before estimating $\mu_2(\mathbf{x},a)$ and $\mu_1(\mathbf{x},a)$, enabling more accurate nuisance estimation with reduced noise. 
In practice, we add a small constant $\epsilon$ (e.g.~$10^{-9}$) to $\hat{\mu}_1(\mathbf{x},a)$ to ensure a strictly positive lower bound.

Under continuous treatment settings, a good estimator should not be dominated by other features, and maintain continuity and smoothness. 
To achieve this, we adopt the varying coefficient model \cite{hastie1993varying,fan1999statistical,chiang2001smoothing} to estimate $\mu_2(\mathbf{x},a)$ and $\mu_1(\mathbf{x},a)$, where the treatment variable $a$ directly determines neural network parameters. 
Specifically, spline basis functions are used to parameterize weights:
\[
w(a) = \sum_{\ell=1}^{L} \alpha_\ell \phi_\ell(a),
\]
where $\alpha_\ell$ are coefficients and $\phi_\ell(\cdot)$ are polynomial basis functions. 
Therefore, $\mu_2(\mathbf{x},a)$ and $\mu_1(\mathbf{x},a)$ are continuous as long as the network activations are continuous.

For the generalized propensity score $\pi(a\mid\mathbf{x})$, the key difficulty is ensuring that the estimator forms a valid density for continuous treatment. 
Following \cite{nie2021vcnet}, we uniformly discretize the interval $[0,1]$ into $B$ grids and estimate $\pi(\cdot|\mathbf{x})$ on these $(B+1)$ grid points as 
$\pi_{grid}(\mathbf{x})=softmax(\mathbf{w}\mathbf{x})\in \mathbb{R}^{B+1}$,
and then the conditional density could be given via linear interpolation,
i.e.,
$\pi(a|\mathbf{x})=\pi_{grid}^{a_1}(\mathbf{x})+B(\pi_{grid}^{a_2}(\mathbf{x})-\pi_{grid}^{a_1}(\mathbf{x}))(a-a_1)$,
where $a_1=\lfloor Ba\rfloor, a_2=\lceil Ba \rceil$.

We estimate the nuisance parameters by minimizing the negative log-likelihood, and the overall training loss is
\[
\mathcal{L}(\hat{\mu}_1, \hat{\mu}_2,\hat{\pi})
= \frac{1}{n} \sum_{i=1}^{n}
\Big[
\ell(y_{1i}, \hat{\mu}_1) 
+ \ell(y_{2i}, \hat{\mu}_2)
- \log \hat{\pi}
\Big].
\]

However, the validity of this naive strategy heavily relies on the function class of the prediction models. 
When using nonparametric methods such as neural networks instead of maximum likelihood, the nuisance estimators often fail to achieve $\sqrt{n}$ convergence and may introduce non-negligible estimation bias.
To better understand the estimation error and motivate the method developed later, consider the Taylor expansion of
\(
\left|\frac{\hat{\mu}_2}{\hat{\mu}_1} - \frac{\mu_2}{\mu_1}\right|
\):

\begin{align*}
\frac{\hat{\mu}_2}{\hat{\mu}_1} - \frac{\mu_2}{\mu_1}
=
&-\frac{\hat{\mu}_2}{\hat{\mu}_1^{2}}(\hat{\mu}_1 - \mu_1)
+ \frac{1}{\hat{\mu}_1}(\hat{\mu}_2 - \mu_2) \\
&+ O\!\left(
(\hat{\mu}_1 - \mu_1)^2
+ \lvert \hat{\mu}_1 - \mu_1 \rvert \, \lvert \hat{\mu}_2 - \mu_2 \rvert
\right).
\end{align*}

The ``first-order bias'' terms $\hat{\mu}_1-\mu_1$ and $\hat{\mu}_2-\mu_2$ imply that the plug-in estimator is not consistent unless all nuisance parameters are consistently estimated. 
The convergence rate is dominated by the slowest‐converging nuisance parameter, which makes the estimator sensitive to model misspecification, especially under flexible nonparametric estimation that trades increased bias for reduced variance.
Eliminating first-order bias is precisely the motivation behind doubly robust estimation \cite{kennedy2024semiparametric},which we will discuss next.

\subsection{The New Doubly Robust Estimator}

A good estimator can achieve lower bias and is robust to model misspecification, i.e., DR property \cite{bang2005doubly,kennedy2024semiparametric,chen2024doubly}. In this section, we propose a new DR estimator tailored to our target estimand,
which remains consistent even if not all nuisance parameters are consistently estimated, and achieves a faster convergence than the nuisance estimators. This is a nontrivial constructive problem, and it forms the theoretical basis that supports our later model design. 
Specifically, we construct the estimator as follows:

\begin{align*}
\hat{\psi}_a^{\mathrm{dr}} 
&= \mathbb{P}_n \!\left[ \frac{\hat{\mu}_2}{\hat{\mu}_1} \right] 
   + \mathbb{P}_n \!\left[ \frac{\delta(A = a)}{\hat{\pi} \, \hat{\mu}_1^2} 
   \big( (Y_2 - \hat{\mu}_2)\hat{\mu}_1 - (Y_1 - \hat{\mu}_1)\hat{\mu}_2 \big) \right]\\
&:=  \psi_a(\hat{\mathbb{P}}) + \mathbb{P}_n ( \phi_a( \hat{\mathbb{P}})).
\end{align*}
where
$
\phi_a(\mathbb{P})=
\frac{
    \delta(A = a)
}{
    \pi \, \mu_1^2
}
\Big[
    \mu_1\big(Y_2 - \mu_2\big)
    + \mu_2\big(Y_1 - \mu_1\big)
\Big]+
\frac{\mu_2}{\mu_1}
- \psi_a(\mathbb{P})
$ is the influence function of the target estimand, which can be viewed as the “derivative over the distribution space”\cite{kennedy2024semiparametric}.
$\hat{\mathbb{P}}$ denotes the fitted data-generating distribution.

To gain intuition about the construction and properties of this estimator, we begin with the following lemma, which can be regarded as the
"distributional Taylor expansion" of the target estimand.
\begin{lemma} \label{lem:vonmises}
$\psi$ confirms the von Mises expansion \cite{kennedy2024semiparametric}:
\[
\psi_a(\hat{\mathbb{P}}) - \psi_a(\mathbb{P})
= -\int \phi_a(\hat{\mathbb{P}})\, d(\mathbb{P}) + R_2(\hat{\mathbb{P}}, \mathbb{P}),
\]
where
\[
\phi_a(\mathbb{P}) = \frac{\delta(A = a)}{\pi \mu_1^2} \big[(Y_2 - \mu_2)\mu_1 - (Y_1 - \mu_1)\mu_2 \big] + \frac{\mu_2}{\mu_1} - \psi_a(\mathbb{P}),
\]

\begin{align*}
R_2(\hat{\mathbb P}, \mathbb P)
&=
\int \frac{1}{\hat\pi \hat\mu_1}(\pi-\hat\pi)(\mu_2-\hat\mu_2)\,d\mathbb P(x) \\
&-
\int \frac{\hat\mu_2}{\hat\pi \hat\mu_1^2}(\pi-\hat\pi)(\mu_1-\hat\mu_1)\,d\mathbb P(x) +
\int R_\mu\,d\mathbb P(x),
\end{align*}

and
\[
|R_\mu|
\le
C\Big(
(\hat\mu_1-\mu_1)^2
+
|\hat\mu_1-\mu_1|\,|\hat\mu_2-\mu_2|
\Big).
\]
\end{lemma}

Lemma \ref{lem:vonmises} suggests how to debias the plug-in estimators, namely by estimating the "first order bias term" $-\int\phi_a(\mathbb{P})$ and subtracting it off, which leads directly to the construction of $\hat{\psi}^{dr}_a$.  This "one-step correction" approach can be viewed as a generalization of Newton methods. Intuitively, if the first-order bias is eliminated, the estimation error will be dominated by 
the second-order remainder $R_2(\hat{\mathbb{P}}, \mathbb{P})$, which only involves 
second-order errors of the nuisance estimators, and thus yields a faster convergence rate. Formally, we analyze the asymptotic properties of $\hat{\psi}^{\mathrm{dr}}_a$ via the decomposition:
\[
\resizebox{1.07\columnwidth}{!}{$
\hat\psi^{dr}_a-\psi_a(\mathbb P)
=
(\mathbb P_n-\mathbb P)\{\phi_a(\mathbb P)\}
+
(\mathbb P_n-\mathbb P)\{\phi_a(\hat{\mathbb P})-\phi_a(\mathbb P)\}
+
R_2(\hat{\mathbb P},\mathbb P)
$}
\]
which leads to the following theorem, based on Lemma \ref{lem:vonmises}.

\begin{theorem} \label{theo:asym}If
\begin{enumerate}
    \item $(\mathbb{P}_n - \mathbb{P})\{\phi_a( \hat{\mathbb{P}}) - \phi_a(\mathbb{P})\} = o_{\mathbb{P}}(1/\sqrt{n})$;
    \item $\|\hat{\pi}(a|\mathbf{x}) - \pi(a|\mathbf{x})\| = o_{\mathbb{P}}(n^{-1/4})$;
    \item $\|\hat{\mu}_1(a, \mathbf{x}) - \mu_1(a, \mathbf{x})\| = o_{\mathbb{P}}(n^{-1/4})$;
    \item $\|\hat{\mu}_2(a, \mathbf{x}) - \mu_2(a, \mathbf{x})\| = o_{\mathbb{P}}(n^{-1/4})$,
\end{enumerate}
then
\[
\hat{\psi}_a^{\mathrm{dr}} - \psi_a
= (\mathbb{P}_n - \mathbb{P})\{\phi_a(\mathbb{P})\} + o_{\mathbb{P}}(1/\sqrt{n}),
\]
and so $\hat{\psi}_a^{\mathrm{dr}}$ is root-n consistent, asymptotically normal by the central limit theorem and Slutsky’s theorem.
\end{theorem}

Condition 1 can be satisfied by cross-fitting \cite{kennedy2024semiparametric}. The conversion rate $o_{\mathbb{P}}(n^{-1/4})$ in conditions 2 and 3 can be satisfied by various estimators including neural networks \cite{chernozhukov2018double, farrell2021deep}. In contrast, root-$n$ consistency is a desirable property that is difficult for 
nonparametric estimators to achieve. Theorem~\ref{theo:asym} shows that 
$\hat{\psi}_a^{\mathrm{dr}}$ attains this property under mild assumptions, 
allowing us to obtain accurate estimates using flexible nuisance estimators. 
This is precisely the asymptotic meaning of ``double robustness'' \cite{kennedy2024semiparametric,chen2024doubly}: 
the final estimator converges faster than any of the nuisance estimators, 
and is therefore robust to model misspecification. (Kernelized IF can be used to improve smoothness under continuous treatment setting\cite{kennedy2017non}, and it can be extended to heterogeneous causal effects estimation by pseudo-outcome regression \cite{kennedy2023towards}. We do not elaborate on these standard practices in detail since they are not specific to the target estimand and is not the focus of our study.)

The theoretical results above lay the foundation for constructing the models 
for causal effect estimation on CVR. However, the direct one-step correction approach may 
be unstable in practice, and can even produce estimates outside the parameter space 
\cite{kennedy2024semiparametric}. In our setting specifically, $\hat{\mu}_1$ appears 
in the denominator of the correction term $\mathbb{P}_n ( \phi_a( \hat{\mathbb{P}}))$. Since click-through rates are often low in practice, the correction term may become extremely large, further amplifying this instability. Furthermore, the cross-fitting required by condition 1, which estimates nuisance parameters using sample-splitting across multiple folds, introduces additional complexity. To address these issues, 
we extend the idea of targeted regularization to our target estimand based on the findings 
above, and develop a more practical framework.

\section{A Practical Framework Based on Targeted Regularization}

As discussed above, the basic one-step correction procedure consists of two stages:
(1) estimating nuisance parameters (which typically involves cross-fitting), 
(2) predicting the correction term and adding it to the plug-in estimator.
In practice, this two-stage procedure has been found to be insufficiently stable.
Inspired by TMLE\cite{van2011targeted}, the widely used DragonNet~\cite{shi2019adapting} introduces \emph{targeted regularization}. The idea is to force the correction term $\mathbb{P}_n(\phi_a(\hat{\mathbb{P}})) \approx 0$ with an extra regularizer, and recover it by learning a low-dimensional parameter $\epsilon$.
Extensive empirical results demonstrate that this approach yields a more stable and practically friendly framework.
Nie et al.~\cite{nie2021vcnet} further extend this framework to continuous treatment settings. However, the above methods are designed for standard causal estimand, and both their methodological construction and theoretical analysis largely rely on the well-studied doubly robust estimator for $\mathbb{E}\big[\mathbb{E}(Y \mid X, A = a)\big]$. In contrast, such an established theoretical foundation is not available for our setting, which is a key challenge of this work.
Building on the preceding theoretical results, we now extend this framework to our target estimand.

For the standard estimand $\psi'=\mathbb{E}\big[\mathbb{E}(Y \mid X, A = a)\big]$, 
the influence function is $\frac{\delta(A = a)}{\pi(a|\mathbf{x})} (Y - \mu(\mathbf{x},a)) +\mu(\mathbf{x},a)- \psi'$, 
and the targeted regularization is constructed as \cite{shi2019adapting,nie2021vcnet}:

\[
\mathcal{R}'(y,\hat{\mu},\hat{\pi},\hat{\epsilon})
= \frac{1}{n} \sum_{i=1}^n
 \left( y_{i} - \hat{\mu}_{i} - \frac{\hat{\epsilon}(a_i)}{\hat{\pi}_{i}} \right)^2.
\]
which is minimized directly during the training phase, $\epsilon$ is an additional learnable parameter. The final estimator becomes:
\[
\hat{\psi}_a^{tr} = \hat{\psi}_a^{plug-in}+\frac{1}{n} \sum_{i=1}^n \frac{\hat{\epsilon}(a)}{\hat{\pi}_i}.
\]
Intuitively, it trains an additional parameter $\epsilon$ to approximate the correction term during the training phase. According to the influence function of our target estimand (as mentioned in the previous section), we adapt targeted regularization as follows:
\[
\mathcal{R}(\hat{\mu}_1,\hat{\mu}_2,\hat{\pi},\hat{\epsilon})
=
\frac{1}{n}\sum_{i=1}^n
\left[
\textstyle\frac{(y_{2i}-\hat{\mu}_{2i})}{\hat{\mu}_{1i}}
-\textstyle\frac{(y_{1i}-\hat{\mu}_{1i})\hat{\mu}_{2i}}
{\hat{\mu}_{1i}^2}
-\textstyle\frac{\hat{\epsilon}(a_i)}{\hat{\pi}_i}
\right]^2.
\]
Combining the loss function of plug-in estimator, we define the final loss function as:
\[
    \mathcal{L}_{TR}(\hat\mu_1,\hat\mu_2, \hat\pi, \hat\epsilon) = \mathcal{L}(\hat\mu_1,\hat\mu_2,\hat\pi) + \beta \mathcal{R}(\hat{\mu}_1, \hat\mu_2, \hat\pi, \hat\epsilon),
\]

Under continuous treatment setting, optimizing $\epsilon(a)$ over the function space of all mappings from $\mathcal{A}$ to $\mathbb{R}$ is not feasible in practice, we follow \cite{nie2021vcnet} to use splines $\{B_k\}_{k=1}^{K_n}$ with $K_n$ basis function $B_k(\cdot)$ to approximate $\epsilon$. 

We now leverage Lemma~\ref{lem:vonmises} and Theorem~\ref{theo:asym} to derive the asymptotic properties of $\hat{\psi}_a^{\mathrm{tr}}$. We first state several standard and mild assumptions commonly adopted in the literature~\cite{nie2021vcnet}:
\begin{enumerate}
    \item[(i)] There exists constant $c > 0$ such that for any $a \in \mathcal{A}$, $\mathbf{x} \in \mathcal{X}$, and $\hat{\pi} \in \mathcal{U}$, we have $1/c \leq \hat{\pi}(a\mid \mathbf{x}) \leq c$, $1/c \leq \pi(a\mid \mathbf x) \leq c$, $\|\mathcal{Q}\|_{\infty} \leq c$ and $\|\mu_1\|_{\infty} \leq c$,$\|\mu_2\|_{\infty} \leq c$.
    \item[(ii)] $\pi, \mu_1,\mu_2 $ and $\hat{\pi},\hat{\mu}_1,\hat{\mu}_2$ have bounded second derivatives for any $\hat{\pi} \in \mathcal{Q}$ $\hat{\mu}_1 \in \mathcal{U}_1$ and $\hat{\mu}_2 \in \mathcal{U}_2$
    \item[(iii)] The Rademacher complexities satisfy $\operatorname{Rad}_n(\mathcal{G}), \operatorname{Rad}_n(\mathcal{Q})$, $\operatorname{Rad}_n(\mathcal{U}_1)$, and $\operatorname{Rad}_n(\mathcal{U}_2) = O(n^{-1/2})$.
    \item[(iv)] $\mathcal{B}_{K_n}$ equals the closed linear span of B-spline with equally spaced knots, fixed degree, and dimension $K_n \asymp n^{1/6}$.
\end{enumerate}

\begin{theorem}\label{thm:targ}
Under the assumptions above and the loss function $\mathcal{L}_{TR}$, we have
\begin{align*}
\|\hat{\psi}_a^{\mathrm{tr}} - \psi_a\|
= & O_p\big(n^{-1/3}\sqrt{\log n} 
+ r_1(n)r_2(n) \\
&+r_1(n)r_3(n) + r_2(n)r_3(n) + r_2(n)^2 \big), \\
\text{where} \quad
&\|\hat{\pi} - \pi\|_{\infty} = O_p(r_1(n)), 
\|\hat{\mu}_1 - \mu_1\|_{\infty} = O_p(r_2(n)), \\
&\|\hat{\mu}_2 - \mu_2\|_{\infty} = O_p(r_3(n)).
\end{align*}
\end{theorem}

The proof of Theorem~\ref{thm:targ} is in the Appendix. The main idea is to decompose $\hat{\psi}_a^{\mathrm{tr}} - \psi_a$ into three terms, among which a key component is exactly the remainder $R_2(\hat{\mathbb{P}}, \mathbb{P})$ derived in the previous section.
Once the first-order bias is controlled by targeted regularization, the convergence rate is dominated by this remainder, allowing us to leverage the results established earlier to obtain the conclusion.

Theorem~\ref{thm:targ} shows that under mild regularity conditions, $\hat{\psi}_a^{\text{tr}}$ maintains the doubly robust property, i.e. it achieves a convergence rate exceeding the individual convergence rates of the nuisance parameter estimators, and is therefore robust to model misspecification, allowing the use of flexible nonparametric estimators, including neural networks.

As discussed in \cite{nie2021vcnet}, adding the extra regularizer does not affect the consistency of nuisance parameter estimation and can substantially improve estimation stability.  The magnitude of the targeted regularization term is controlled by a hyperparameter $\beta$, which allows it to be kept on a comparable scale to the main loss function.
Intuitively, this approach replaces the “hard” correction in the original DR estimator with a “soft” correction via regularization. In practice, stability can be further enhanced by controlling the learning rate of $\epsilon$ or by introducing additional regularization on $\epsilon$.
Moreover, this end-to-end method does not involve cross-fitting, which leads to substantially higher efficiency in realistic settings involving complex models. Therefore,
we adopt this approach as the primary modeling framework.


As discussed above, obtaining the final estimator requires estimating three nuisance parameters: $\pi$, $\mu_1$, and $\mu_2$, where $\mu_1$ corresponds exactly to the click probability.
A natural idea is that once $\hat{\pi}$ and $\hat{\mu}_1$ are obtained, one can immediately construct an estimator for the causal effect of the click-through rate (CTR):
\[
\psi^{ctr}_a = \mathbb{E}\!\left[\mathbb{E}(Y_1 \mid \mathbf{X}, A = a)\right].
\]
Motivated by this observation, we design a \emph{multi-task} model to jointly estimate the causal effects of CTR and CVR, denoted by $\hat{\psi}^{ctr}_a$ and $\hat{\psi}^{cvr}_a$.
The overall loss function becomes:
\[
\mathcal{L}_{multi-task} = \mathcal{L}_1 + \mathcal{L}_2 + \alpha\mathcal{L}_{\pi} + \beta(\mathcal{R}_1 + \mathcal{R}_2)
\]
where
\[
\mathcal{L}_{\pi}
= \frac{1}{n} \sum_{i=1}^n
- \log\big(\hat{\pi}(a_i \mid \mathbf{x}_i)\big),
\]
\[
\mathcal{L}_1
= \frac{1}{n} \sum_{i=1}^n
\Big(
- y_{1i} \log(\hat{\mu}_{1i})
- (1 - y_{1i}) \log(1 - \hat{\mu}_{1i})
\Big),
\]
\[
\mathcal{L}_2
= \frac{1}{n} \sum_{i=1}^n
\Big(
- y_{2i} \log(\hat{\mu}_{2i})
- (1 - y_{2i}) \log(1 - \hat{\mu}_{2i})
\Big),
\]
\[
\mathcal{R}_1
= \frac{1}{n} \sum_{i=1}^n
\left(
y_{1i} - \hat{\mu}_{1i}
- \frac{\hat{\epsilon}_1(a_i)}{\hat{\pi}_i}
\right)^2,
\]
\[
\mathcal{R}_2
= \frac{1}{n} \sum_{i=1}^n
\left[
\frac{y_{2i} - \hat{\mu}_{2i}}{\hat{\mu}_{1i}}
- \frac{(y_{1i} - \hat{\mu}_{1i}) \hat{\mu}_{2i}}{\hat{\mu}_{1i}^2}
- \frac{\hat{\epsilon}_2(a_i)}{\hat{\pi}_i}
\right]^2.
\]

The final estimators are given by
\[
\hat{\psi}^{ctr}_a
= \frac{1}{n} \sum_{i=1}^n
[
\hat{\mu}_{1i}
+ \frac{\hat{\epsilon}_1(a)}{\hat{\pi}_i}
],
\hat{\psi}^{cvr}_a
= \frac{1}{n} \sum_{i=1}^n
[
\frac{\hat{\mu}_{2i}}{\hat{\mu}_{1i}}
+ \frac{\hat{\epsilon}_2(a)}{\hat{\pi}_i}
].
\]

Figure~\ref{fig:model} illustrates the overall model architecture.
To explicitly model the dependency $\mu_2 < \mu_1$, we construct the prediction of $\mu_2$ as
$\hat{\mu}_2 = \hat{\mu}_1 \times \tilde{\mu}_2$, where $\tilde{\mu}_2 \in (0,1)$.
Motivated by the TMLE, targeted regularization should not affect the estimation of $\hat{\pi}$.
Accordingly, gradients from the targeted regularization terms $\mathcal{R}_1$ and $\mathcal{R}_2$ are blocked from propagating to $\hat{\pi}$, which improves the accuracy of propensity score estimation.

\begin{figure}[htbp]
    \centering
    \includegraphics[width=0.4\textwidth]{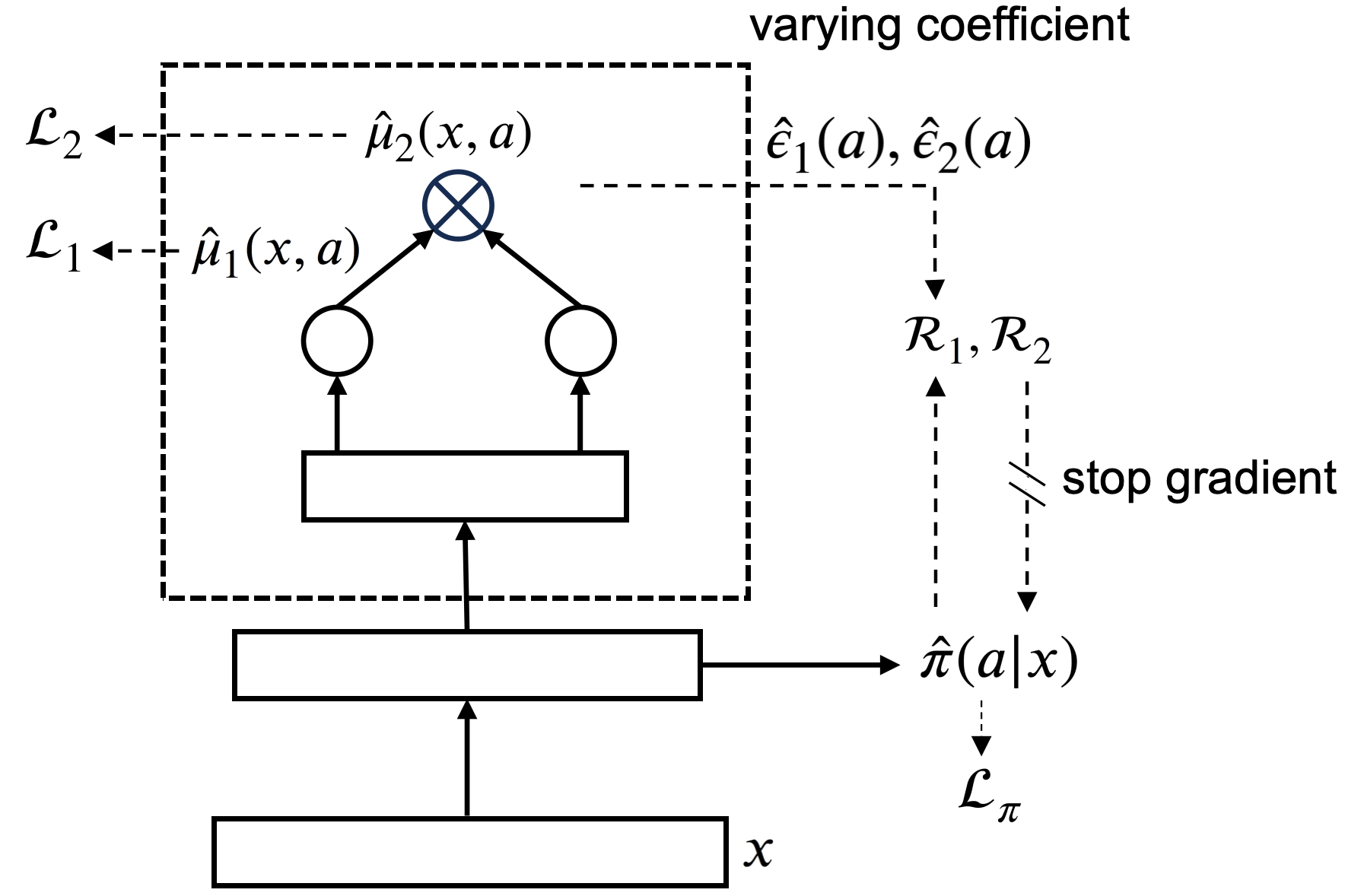} 
    \caption{Model architecture.}
    \label{fig:model}
\end{figure}

\section{Experiments}
Our experiments address the following questions:
\begin{itemize}
  \item \textbf{RQ1:} 
  For our target estimand, does the proposed method outperform existing baseline models?

  \item \textbf{RQ2:} 
  Does the targeted regularization lead to improved estimation performance?

  \item \textbf{RQ3:} 
  Is the proposed method stable under different hyperparameter settings?

  \item \textbf{RQ4:} 
  Is it necessary to redesign the estimation framework, or can existing loss-debiasing approach developed for CVR prediction be directly combined with standard causal effect estimation methods?
\end{itemize}
\subsection{Setup}
\subsubsection{Synthetic and Semi-synthetic Data Generation}

To obtain the true value of the target estimand, we generate synthetic datasets following \cite{nie2021vcnet}, each dataset consists of 3{,}000 training samples and 1{,}000 test samples. Further, we reuse the 498-dimensional covariates $X$ of the real-word dataset News to generate the assigned treatments and their corresponding outcomes following \cite{schwab2020learning, bica2020estimating, nie2021vcnet}, and each dataset is randomly split into a training set (67\%) and a test set (33\%). Hyperparameters are tuned on 20 extra synthetic datasets disjoint from the test samples. We provide details of data generation in Appendix A.4.

Table~\ref{tab:data_sum} reports the means and standard deviations of the observed outcomes $Y_1$ and $Y_2$ on both simulated and semi-synthetic datasets (with one representative sample selected).
As shown, the mean of $Y_1$ is below $0.5$, and overall $P(Y_2 = 1 \mid Y_1 = 1) < P(Y_1 = 1)$, which is consistent with real-world scenarios.
\setlength{\tabcolsep}{2pt}
\begin{table}[htbp]
\centering
\caption{Summary Statistics of Outcomes}
\label{tab:data_sum}
\begin{tabular}{lccc|ccc}
\toprule
 & \multicolumn{3}{c|}{Synthetic Data}
 & \multicolumn{3}{c}{News} \\
\cmidrule(lr){2-4} \cmidrule(lr){5-7}
 & $Y_1$ & $Y_2$ & $Y_2 \mid Y_1=1$
 & $Y_1$ & $Y_2$ & $Y_2 \mid Y_1=1$ \\
\midrule
Mean
 & 0.301 & 0.052 & 0.170
 & 0.347 & 0.080 & 0.231 \\
Std.\ Dev.
 & 0.459 & 0.222 & 0.375
 & 0.476 & 0.272 & 0.421 \\
\bottomrule
\end{tabular}
\end{table}

\subsubsection{Baselines and Settings}
We compare our method with: 
  
  
  
  


(1) \textbf{DragonNet}~\cite{shi2019adapting}:
Introduces targeted regularization for causal effect estimation, requiring discretization for continuous treatments;
(2) \textbf{VCNet}~\cite{nie2021vcnet}:
Extends DragonNet to continuous treatments via varying-coefficient networks and spline functions;
(3) \textbf{DRNet}~\cite{schwab2020learning}:
Estimates continuous treatment effects by partitioning the treatment space while preserving treatment information through hierarchical inputs;
(4) \textbf{TARNet}~\cite{shalit2017estimating}:
Learns a shared representation across treatment groups;
(5) \textbf{Causal Forest}~\cite{athey2019generalized}:
Generalizes random forests to estimate causal effects;
(6) \textbf{ECUP}~\cite{huang2024entire}: 
A recent method focusing on treatment effect estimation for CVR, and leverages attention mechanisms to enhance the interaction between the treatment and other features;
(7) \textbf{DR Estimator}:
The basic doubly robust estimator proposed in Section~4 with cross-fitting;
(8) \textbf{Ours (w/o.\ t-reg)}:
Our method without the designed targeted regularization, which reduces to a plug-in estimator;

Since we focus on the underlying methodological framework, our baselines exclude approaches that solely target model architecture design (e.g.,~\cite{liu2023explicit}). Our backbone network can be replaced by arbitrary architectures; here, we unify the model capacity (e.g., the number of layers and hidden layer dimensions) across all methods for fair comparison. Some baselines require adaptation for continuous treatment, and we follow the implementation in \cite{nie2021vcnet}.

Each model is trained for 600 epochs. 
The loss weights $\alpha$ and $\beta$
are set to 0.5 and 1, respectively. The learning rate and the number of hidden units are selected via grid search within 
$\{10^{-5},10^{-4},5\times10^{-4},10^{-3}\}$ and $\{8,32,128\}$
(we observed that model performance is insensitive to the scale of model capacity). 
All other settings are kept the same as in \cite{nie2021vcnet}. 
We replace the MSE loss in baseline models with cross-entropy loss since the outcomes are binary.


\subsubsection{Metric}
We use the average mean squared error (AMSE) as the primary evaluation metric, 
which is consistent with prior studies on causal effect estimation with continuous treatment \cite{nie2021vcnet}.
\[
\text{AMSE} = \frac{1}{S} \sum_{s=1}^{S} \int_{A} \big[\hat{\psi}^s_a - \psi_a\big]^2 \pi(a) \, da,
\]
where $\hat{\psi}^s_a$ denotes the estimate obtained in the $s$-th simulation.

\subsection{Result and Analysis}
\subsubsection{Overall Performance and ablation study}

Table \ref{tab:performance} presents the AMSE of our proposed multi-task framework and baseline models on synthetic data and News, where the ``CVR task'' and ``CTR task'' correspond to the estimation of $\hat{\psi}^{\mathrm{cvr}}_a$ and $\hat{\psi}^{\mathrm{ctr}}_a$, respectively.
Several key observations can be drawn from the results:

\begin{itemize}
  \item Overall, our proposed method significantly outperforms all baselines on the CVR task across different datasets, demonstrating strong effectiveness. The performance on the CTR task remains desirable when additional task and regularization terms are introduced, validating the feasibility of jointly estimating the causal effects on CTR and CVR.
  \item Without the designed targeted regularization, the performance of our method drops significantly, demonstrating its effectiveness and necessity.
  \item The performance of the basic DR estimator is also competitive; however, it is inferior to our final proposed method, which is consistent with the discussion in previous sections and findings in \cite{shi2019adapting,nie2021vcnet}.

\end{itemize}

\begin{table*}[t]
\centering
\caption{Results on Synthetic Data and News. Reported: AMSE (mean $\pm$ standard deviation) over 20 simulations.}
\label{tab:performance}

\begin{tabular*}{0.75\textwidth}{@{\extracolsep{\fill}}lcc|cc}
\toprule
 & \multicolumn{2}{c|}{Synthetic Data} & \multicolumn{2}{c}{News} \\
\cmidrule(lr){2-3} \cmidrule(lr){4-5}
 & $\textbf{CVR Task}$ & CTR Task & $\textbf{CVR Task}$ & CTR Task \\
\midrule
DragonNet
& $0.01130_{\pm 0.00359}$ & $0.00115_{\pm 0.00032}$
& $0.00706_{\pm 0.00157}$ & $0.00102_{\pm 0.00025}$ \\

VCNet
& $0.00935_{\pm 0.00287}$ & $0.00093_{\pm 0.00031}$
& $0.00325_{\pm 0.00184}$ & $0.00040_{\pm 0.00017}$ \\

DRNet
& $0.01375_{\pm 0.00392}$ & $0.00423_{\pm 0.00173}$
& $0.01112_{\pm 0.00320}$ & $0.00178_{\pm 0.00083}$ \\

TARNet
& $0.01529_{\pm 0.00320}$ & $0.00621_{\pm 0.00147}$
& $0.01163_{\pm 0.00224}$ & $0.00149_{\pm 0.00039}$ \\

Causal Forest
& $0.01031_{\pm 0.00172}$ & $0.00878_{\pm 0.00080}$
& $0.00341_{\pm 0.00055}$ & $0.00132_{\pm 0.00014}$ \\

ECUP
& $0.00816_{\pm 0.00209}$ & $0.00145_{\pm 0.00062}$
& $0.00233_{\pm 0.00057}$ & $0.00061_{\pm 0.00032}$ \\

DR Estimator
& $0.00511_{\pm 0.00298}$ & $0.00310_{\pm 0.00153}$
& $0.00321_{\pm 0.00202}$ & $0.00129_{\pm 0.00076}$ \\

\midrule
Ours(w/o. t-reg)
& $0.00981_{\pm 0.00346}$ & $0.00473_{\pm 0.00122}$
& $0.00698_{\pm 0.00249}$ & $0.00132_{\pm 0.00076}$ \\

Ours
& $\pmb{0.00248_{\pm 0.00087}}$ & $0.00096_{\pm 0.00031}$
& $\pmb{0.00101_{\pm 0.00052}}$ & $0.00047_{\pm 0.00025}$ \\
\bottomrule
\end{tabular*}
\end{table*}

\subsubsection{Hyperparameter Sensitivity Analysis}
We conduct a sensitivity analysis by varying the propensity score loss weight $\alpha$ and the targeted regularization weight $\beta$, and evaluate the resulting AMSE on the CVR task. As shown in the Figure ~\ref{fig:sens_alpha}, our method exhibits strong stability across a wide range of $\alpha$ values, the AMSE consistently remains substantially lower than that of the ablation study, demonstrating the robustness of our method. In contrast, Figure~\ref{fig:sens_beta} reveals a monotonic degradation in performance as the targeted regularization weight $\beta$ decreases, which further verifies the effectiveness and necessity of this component.

\begin{figure}[htbp]
    \centering
    \begin{subfigure}{0.49\linewidth}
        \centering
        \includegraphics[width=\linewidth]{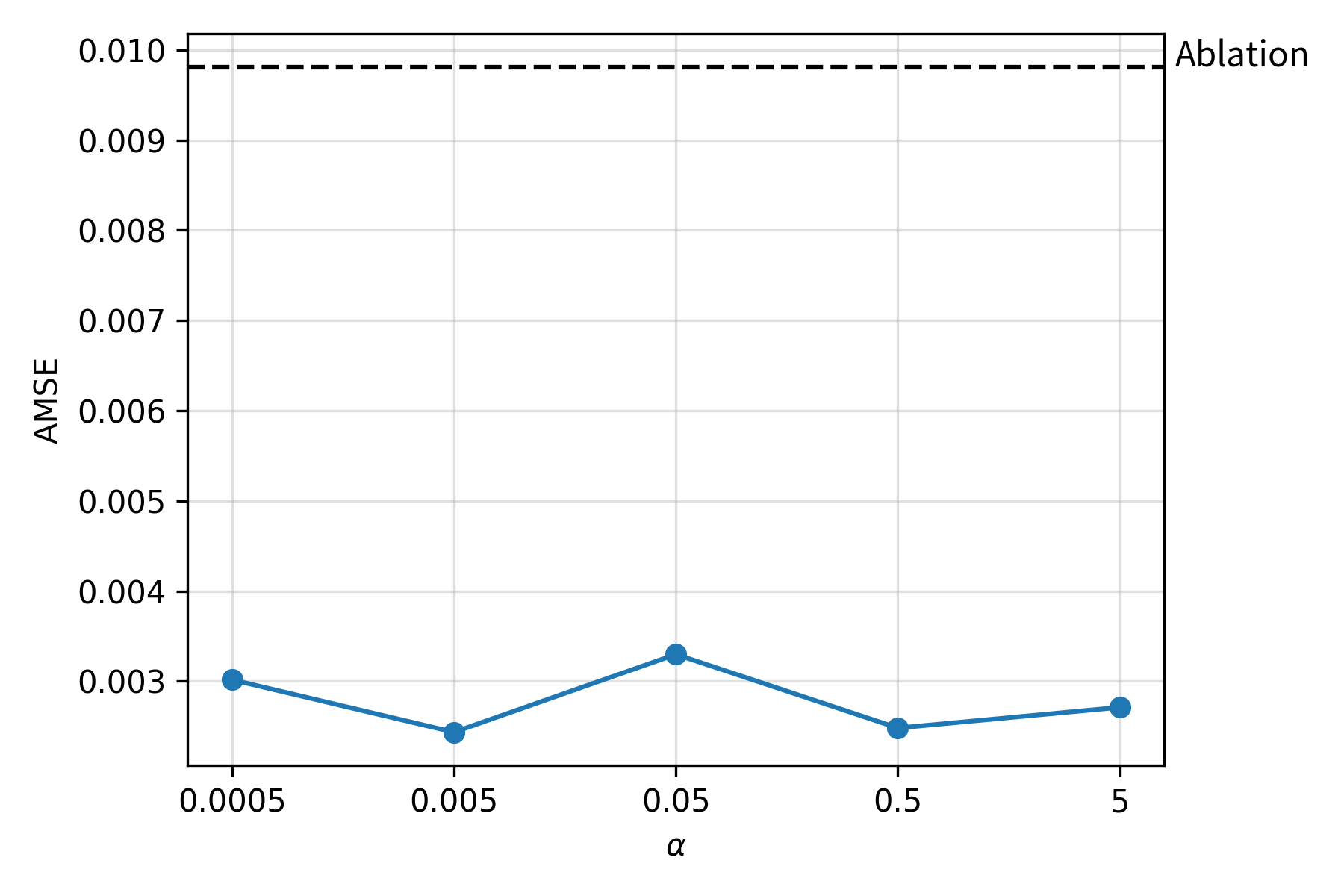}
        \caption{Synthetic Data}
    \end{subfigure}
    \hfill
    \begin{subfigure}{0.49\linewidth}
        \centering
        \includegraphics[width=\linewidth]{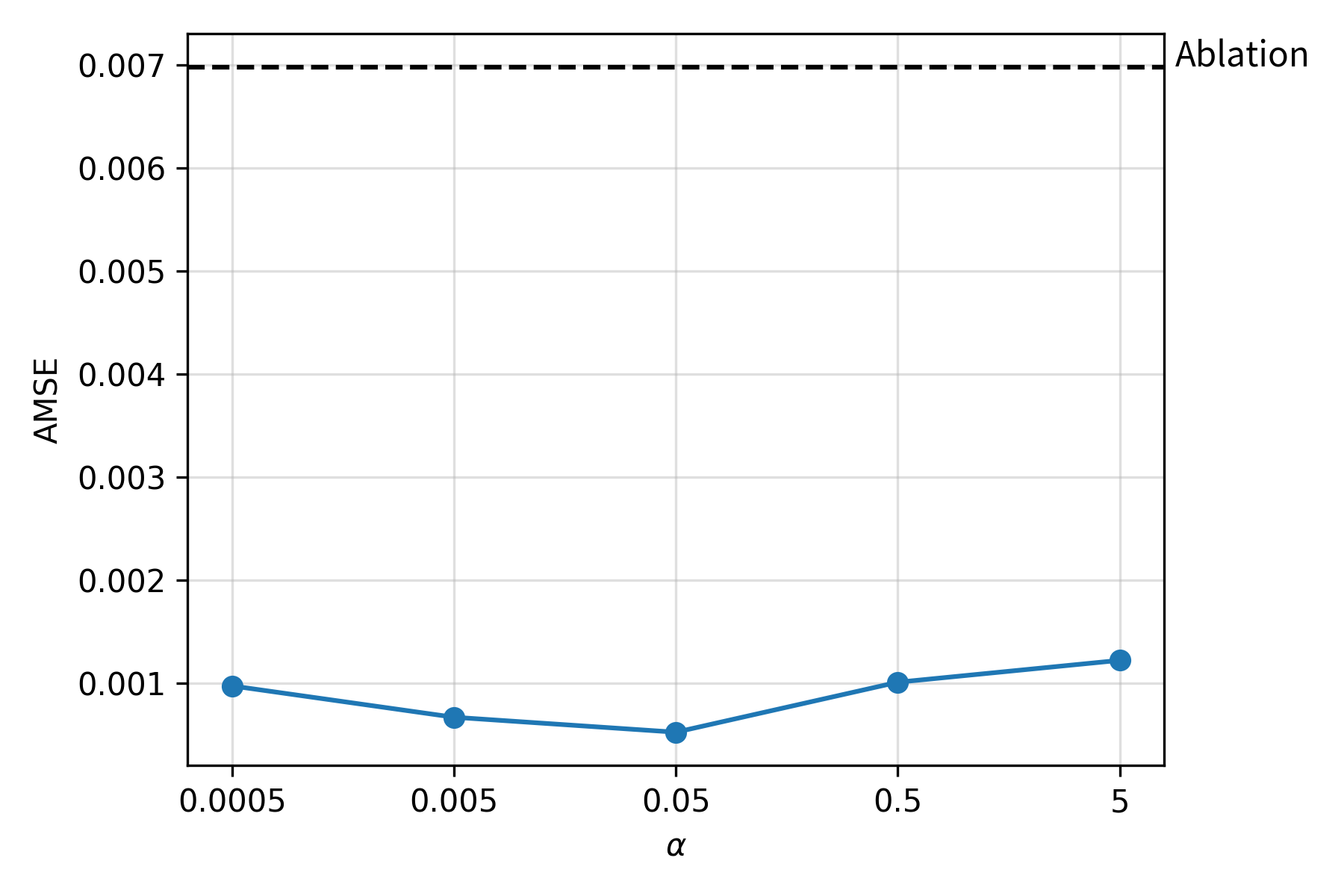}
        \caption{News}
    \end{subfigure}

    \caption{AMSE under different settings of $\alpha$.}
    \label{fig:sens_alpha}
\end{figure}

\begin{figure}[htbp]
    \centering
    \begin{subfigure}{0.49\linewidth}
        \centering
        \includegraphics[width=\linewidth]{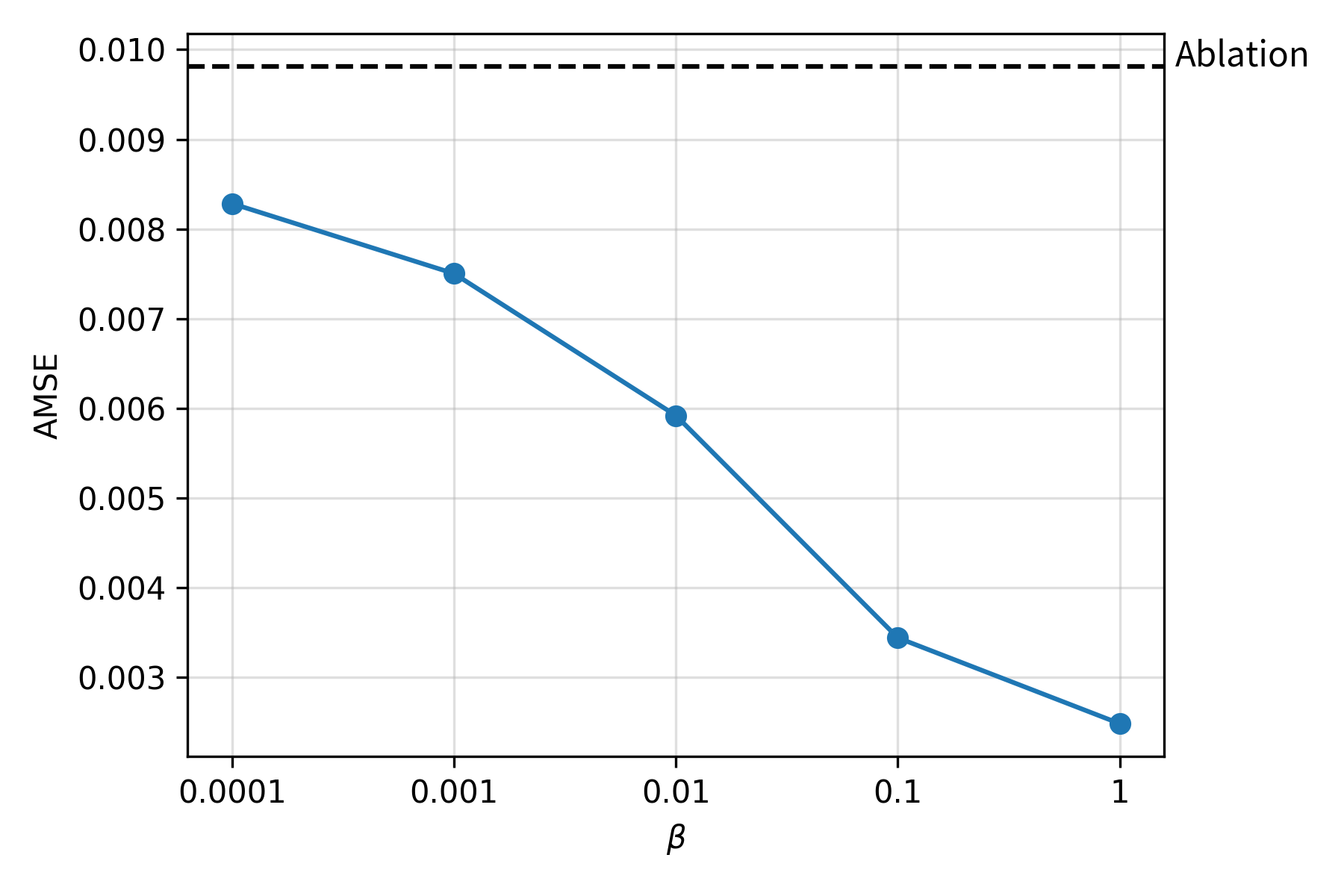}
        \caption{Synthetic Data}
    \end{subfigure}
    \hfill
    \begin{subfigure}{0.49\linewidth}
        \centering
        \includegraphics[width=\linewidth]{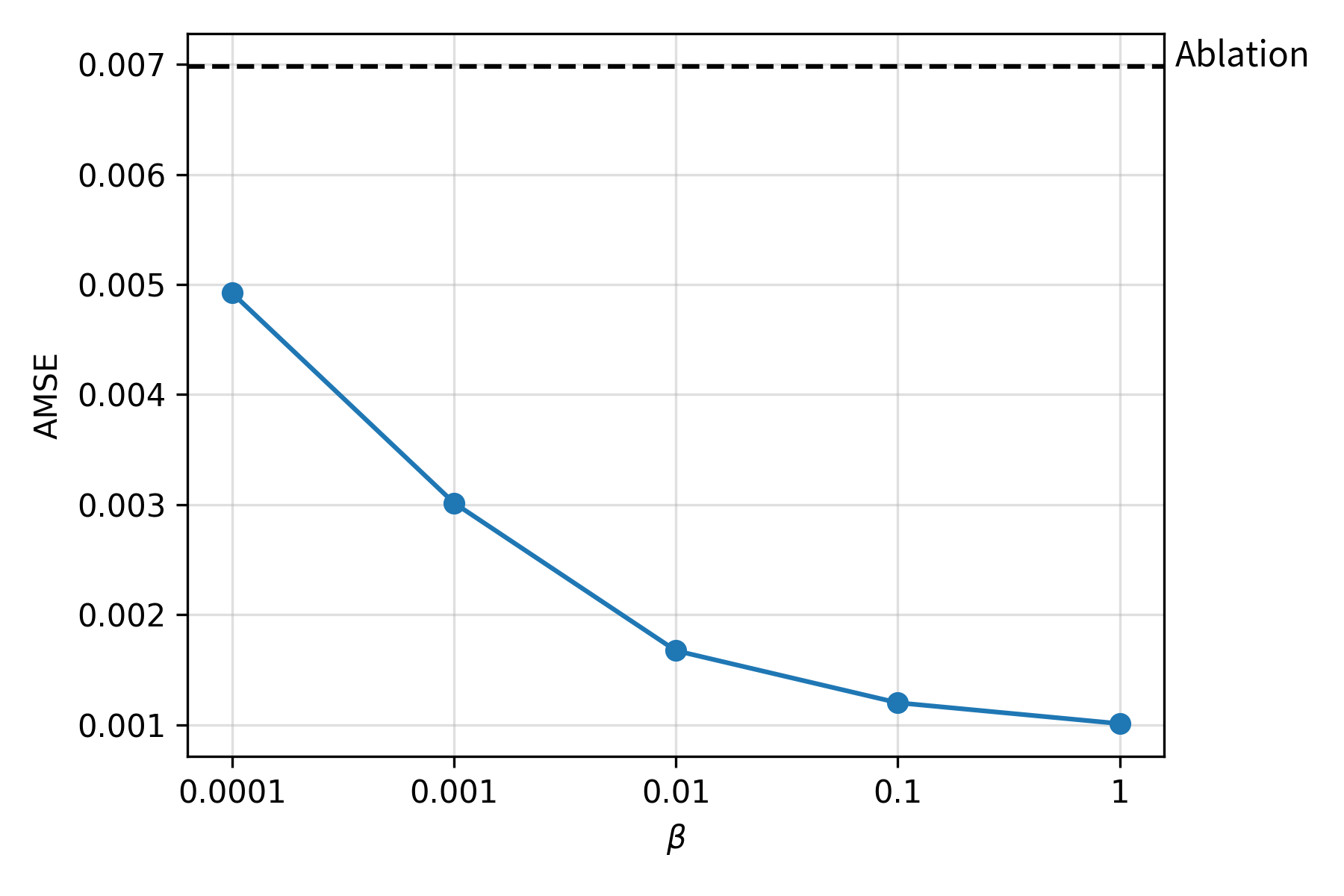}
        \caption{News}
    \end{subfigure}

    \caption{AMSE under different settings of $\beta$.}
    \label{fig:sens_beta}
\end{figure}

\subsubsection{Comparison with direct loss debiasing}
While our proposed method directly targets the final estimand, a well-established line of prior work that specifically addresses the sample selection bias issue in CVR prediction is based on the concept of the \emph{ideal loss} ~\cite{wang2019doubly,zhang2020large,guo2021enhanced,dai2022generalized,li2022stabledr,li2022tdr}. 
The core idea of this approach is to train models using only clicked samples, while constructing an unbiased estimator of the loss over the full population, typically via reweighting, and then optimize model parameters using this debiased loss.
Existing studies primarily apply this methodology to prediction problems, rather than causal effect estimation. 
This naturally raises a question: can the loss-debiasing idea be straightforwardly combined with standard causal inference methods? 
Is it necessary to construct a new estimator  tailored to the estimand of this work?

To answer this question, we propose an alternative approach that combines loss debiasing with standard causal effect estimation, and compare it with our previously proposed method. 
Specifically, when using standard causal estimators directly, selection bias mainly arises from the following two components:
\begin{itemize}
    \item The prediction loss for $\mu'_2=\mathbb{E}(Y_2|Y_1=1,X=x,A=a)$, which is similar to that in CVR prediction:
    \[
    \mathcal{L}^{biased}_2 
    = |\mathcal{C}|^{-1} \sum_{i \in \mathcal{C}} \Big(
- y_{2i} \log(\hat{\mu}'_{2i})
- (1 - y_{2i}) \log(1 - \hat{\mu}'_{2i})
\Big);
    \]
    \item The standard targeted regularization term\cite{nie2021vcnet}:
    \[
    \mathcal{R}^{biased}_2 
    = |\mathcal{C}|^{-1} \sum_{i \in \mathcal{C}} 
\left(
y_{2i} - \hat{\mu}'_{2i}
- \frac{\hat{\epsilon}'_2(a_i)}{\hat{\pi}_i}
\right)^2.
    \]
\end{itemize}
where $\mathcal{C}$ denotes the clicked set. Therefore, a straightforward solution is to get unbiased estimates of these two terms for backpropagation. 
Here, we adopt the widely used and most convenient inverse propensity scoring (IPS) method~\cite{schnabel2016recommendations}, and replace $\mathcal{L}_2$ and $\mathcal{R}_2$ in our proposed method with  $\mathcal{L}'_2$ and $\mathcal{R}'_2$, respectively:
\[
\mathcal{L}'_2 
= |\mathcal{C}|^{-1} \sum_{i \in \mathcal{C}} \frac{1}{\hat{\mu}_{1i}} \Big(
- y_{2i} \log(\hat{\mu}'_{2i})
- (1 - y_{2i}) \log(1 - \hat{\mu}'_{2i})
\Big);
\]
    \[
    \mathcal{R}'_2 
    = |\mathcal{C}|^{-1} \sum_{i \in \mathcal{C}} 
\frac{1}{\hat{\mu}_{1i}}\left(
y_{2i} - \hat{\mu}'_{2i}
- \frac{\hat{\epsilon}'_2(a_i)}{\hat{\pi}_i}
\right)^2,
    \]
where gradients are blocked from propagating through $\hat{\mu}_1$ in both $\mathcal{L}'_2$ and $\mathcal{R}'_2$. The final estimator becomes:
\[
\hat{\psi}'^{cvr}_a = \frac{1}{n} \sum_{i=1}^n
\left(
\hat{\mu}'_{2i}
+ \frac{\hat{\epsilon}'_2(a)}{\hat{\pi}_i}
\right).
\]

Table~\ref{tab:alt} compares the performance of this alternative approach with our previously proposed method on the CVR task (the method and performance on the CTR task remain unchanged). 
It can be observed that the alternative approach consistently outperforms the plug-in estimator without targeted regularization across different datasets, and also achieves better results than most baselines reported in Tables~\ref{tab:performance}. 
This is because it also leverages a targeted regularization to reduce estimation bias, and obtains an unbiased estimate of the loss terms. 
However, the alternative approach performs noticeably worse than our proposed method. 
The key reason is that it focuses on eliminating the estimation bias of the loss terms, rather than directly targeting the final estimand. 
There is no theoretical guarantee that unbiased estimation of the loss terms necessarily leads to accurate estimation of the final estimand. 
Even if the "unbiased loss" is well approximated over the entire parameter space, its minimizer may still be poorly estimated. 
Therefore, it is necessary to directly construct estimation procedures tailored to the target estimand with sound theoretical properties, rather than simply combining existing methods.

\begin{table}[htbp]
\centering
\caption{Comparison with the Alternative Method. 
Reported: AMSE (mean $\pm$ standard deviation) over 20 simulations.}
\label{tab:alt}
\begin{tabular}{lcc}
\toprule
Model & Synthetic Data & News \\
\midrule
Alternative Method
& $0.00494_{\pm 0.00161}$
& $0.00323_{\pm 0.00260}$ \\

Ours
& $\textbf{0.00248}_{\pm 0.00087}$
& $\textbf{0.00101}_{\pm 0.00052}$ \\

w/o. t-reg
& $0.00981_{\pm 0.00346}$
& $0.00698_{\pm 0.00249}$ \\
\bottomrule
\end{tabular}
\end{table}

\subsection{Performance on Real-World Dataset}
To further demonstrate the practical applicability of our method, we conduct experiments and ablation studies on a large-scale real-world dataset CRITEO-UPLIFTv2~\cite{diemert2021large}. 
It contains approximately 13 million samples with 12 features, a binary treatment variable, and two binary outcomes: visit and conversion. 
We randomly sample 10\% of the data for our experiments, and further split it into a training, validation and a test set with a ratio of 60\%, 20\% and 20\%.

Since the publicly available datasets contain only binary treatment, we adapt the varying coefficient network to separately predict nuisance parameters $\mu_1$ and $\mu_2$ for the two treatment groups, and removed the baselines specifically designed for continuous treatments.
Moreover, since the true value of $\psi$ cannot be obtained, commonly used evaluation metrics are AUUC and QINI, which require defining the estimand at the individual level. Our previous theoretical analysis are established at an averaged level, which is consistent with the settings adopted in widely used methods such as DragonNet~\cite{shi2019adapting} and VCNet~\cite{nie2021vcnet}. 
At the individual level, a natural approach is to remove the final averaging step in $\hat{\psi}^{tr}$.
In summary, due to the limitations of publicly available datasets, our method may be underestimated. The preceding experiments provide a more comprehensive and rigorous evaluation that better aligns with the settings of our work. This section serves only as a supplementary reference, primarily to further address potential concerns regarding the practical applicability of our method in real-world scenarios.


As shown in Table~\ref{tab:criteo}, our method still achieves significantly better performance in estimating the causal effect on CVR than the other models. The ablation study once again demonstrates that the targeted regularization term plays a critical role in improving the performance. The cumulative gain curves shown in Fig.~\ref{fig:gain} provide an intuitive illustration of the performance gains brought by the targeted regularization and the doubly robust design.
These results further validate the effectiveness and practical utility of our method in real-world scenarios.






\begin{table}[htbp]
\centering
\caption{Performance on the CVR task (CRITEO).}
\label{tab:criteo}

\begin{tabular}{lcc}
\toprule
& AUUC & QINI \\
\midrule
DragonNet
& 0.02216 & 0.04483 \\

TARNet
& 0.02031 & 0.04129 \\

T-Learner
& 0.01798 & 0.03467 \\

Causal Forest
& 0.00405 & 0.00798 \\

ECUP
& 0.02765 & 0.05253 \\
\midrule
Ours
& \textbf{0.03208} & \textbf{0.06014} \\

Ours (w/o t-reg)
& 0.01017 & 0.01950 \\
\bottomrule
\end{tabular}

\parbox{\linewidth}{\footnotesize
\noindent
Following previous work using this dataset~\cite{liu2023explicit}, we compute
AUUC and QINI using \texttt{uplift\_auc\_score} and \texttt{qini\_auc\_score}
from \texttt{sklearn}, both of which are normalized by their theoretical
maximum values.
}

\end{table}

\section{Conclusion}
In this work, we formulate the CVR causal effect estimation problem from a semiparametric perspective, and propose an estimation framework with both theoretical guarantees and practical value. Specifically, we construct a new doubly robust estimator tailored to the target estimand, and show that it achieves $\sqrt{n}-$ consistency under mild conditions. Building on these theoretical findings, we develop a practical framework based on targeted regularization. Extensive experiments demonstrate the strong performance of the proposed method and the effectiveness of targeted regularization, as well as its advantages over the indirect loss-debiasing approach.

\begin{figure}[H]
    \centering
    \includegraphics[width=0.4\textwidth]{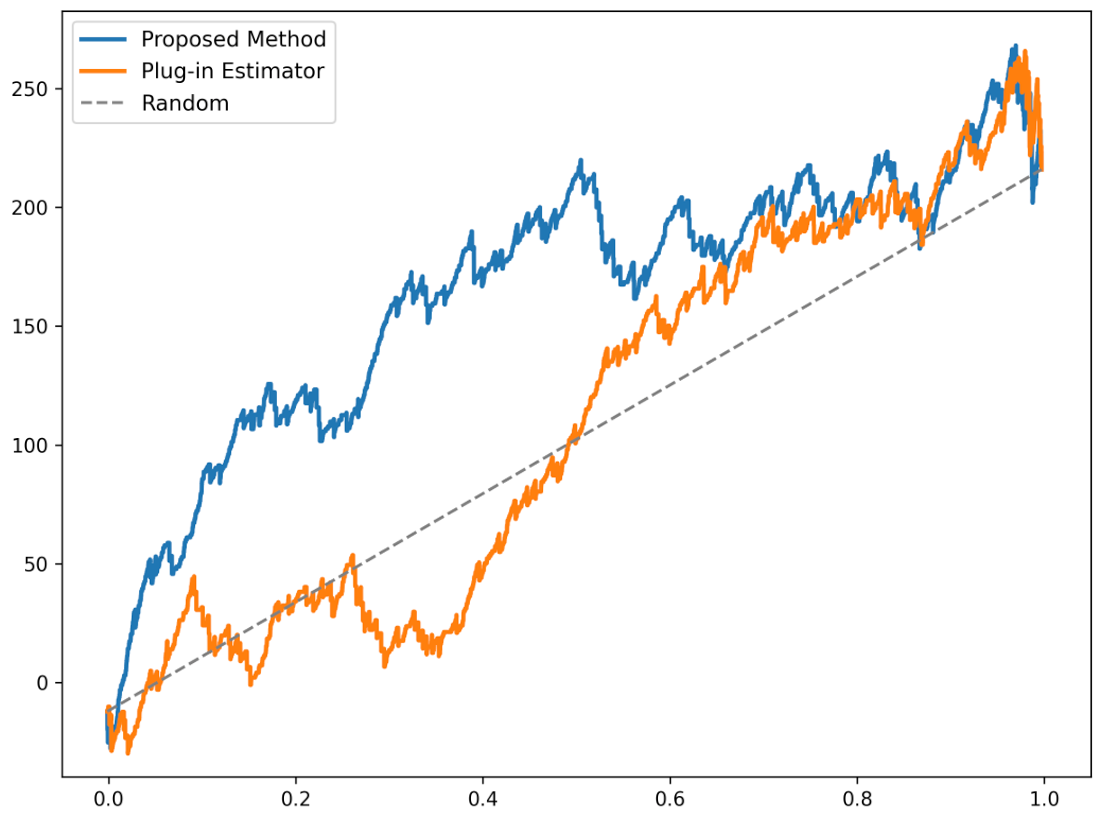} 
    \caption{Cumulative Gain Curve.}
    \label{fig:gain}
\end{figure}

\section{Impact Statement}
This paper proposes a general framework for estimating causal effects of chain-structured outcomes such as CVR. Our method can be applied to a wide range of applications, including decision-making in e-commerce and online advertising scenarios.


\bibliographystyle{ACM-Reference-Format}
\bibliography{references}

\appendix

\section{Appendix}

\subsection{Proof of Lemma \ref{lem:vonmises}}

\begin{proof}
\begin{align*}
&R_2(\hat{\mathbb{P}},\mathbb{P}) 
= \psi_a(\hat{\mathbb{P}})-\psi_a(\mathbb{P}) + \int \phi_a(\hat{\mathbb{P}})\,d\mathbb{P} \\
&= \int \Bigg\{ 
\frac{\delta(A=a)}{\hat{\pi}(A|\mathbf{X)})}
\Big[
\frac{Y_2-\hat{\mu}_2(\mathbf{X},A)}{\hat{\mu}_1(\mathbf{X},A)} \\
& \quad -\frac{\hat{\mu}_2(\mathbf{X},A) (Y_1-\hat{\mu}_1(\mathbf{X},A))}{\hat{\mu}_1^2(\mathbf{X},A)}\Big]
+ \frac{\hat{\mu}_2(\mathbf{X},A)}{\hat{\mu}_1(\mathbf{X},A)}
\Bigg\} d\mathbb{P}
- \psi_a(\mathbb{P}) \\
&= \int_{\mathcal{X}}\int_{\mathcal{Y}_1}
-\frac{\hat{\mu}_2 (y_1-\hat{\mu}_1)}{\hat{\pi}\hat{\mu}_1^2}
p(y_1,\mathbf{x},a)\,dy_1\,d\mathbf{x} \\
&\quad+ \int_{\mathcal{X}}\int_{\mathcal{Y}_2}
\frac{y_2-\hat{\mu}_2}{\hat{\pi}\hat{\mu}_1}
p(y_2,\mathbf{x},a)\,dy_2\,d\mathbf{x} \\
&\quad
+ \int \frac{\hat{\mu}_2}{\hat{\mu}_1}\,d\mathbb{P}(\mathbf{x})
- \psi_a(\mathbb{P}) \\
&= \int_{\mathcal{X}}\int_{\mathcal{Y}_1}
\frac{-\hat{\mu}_2 (y_1-\hat{\mu}_1)}{\hat{\pi}\hat{\mu}_1^2}
p(y_1\mid\mathbf{x},a)\,\pi(a\mid\mathbf{x})\,p(\mathbf{x})
\,dy_1\,d\mathbf{x} \\
& \quad
+ \int_{\mathcal{X}}\int_{\mathcal{Y}_2}
\frac{y_2-\hat{\mu}_2}{\hat{\pi}\hat{\mu}_1}
p(y_2\mid\mathbf{x},a)\,\pi(a\mid\mathbf{x})\,p(\mathbf{x})
\,dy_2\,d\mathbf{x} \\
&\quad
+ \int \frac{\hat{\mu}_2}{\hat{\mu}_1}\,d\mathbb{P}(\mathbf{x})
- \psi_a(\mathbb{P}) \\
&= \int_{\mathcal{X}}
\frac{\pi}{\hat{\pi}}
\int_{\mathcal{Y}_1}
\frac{-\hat{\mu}_2 (y_1-\hat{\mu}_1)}{\hat{\mu}_1^2}
p(y_1\mid\mathbf{x},a)\,dy_1
\,d\mathbb{P}(\mathbf{x}) \\
&\quad
+ \int_{\mathcal{X}}
\frac{\pi}{\hat{\pi}}
\int_{\mathcal{Y}_2}
\frac{y_2-\hat{\mu}_2}{\hat{\mu}_1}
p(y_2\mid\mathbf{x},a)\,dy_2
\,d\mathbb{P}(\mathbf{x}) \\
&\quad + \int \frac{\hat{\mu}_2}{\hat{\mu}_1}\,d\mathbb{P}(\mathbf{x})
- \psi_a(\mathbb{P}) \\
&= \int
\frac{\pi}{\hat{\pi}}
\frac{\mu_2-\hat{\mu}_2}{\hat{\mu}_1}
\,d\mathbb{P}(\mathbf{x})
+ \int
\frac{\pi}{\hat{\pi}}
\Bigg(
-\frac{\hat{\mu}_2}{\hat{\mu}_1^2}
\Bigg)
(\mu_1-\hat{\mu}_1)
\,d\mathbb{P}(\mathbf{x}) \\
&\quad- \int
\Bigg(
\frac{\mu_2}{\mu_1}
- \frac{\hat{\mu}_2}{\hat{\mu}_1}
\Bigg)
\,d\mathbb{P}(\mathbf{x}).
\end{align*}

Here $p(\cdot)$ denotes the corresponding probability density.
Define $f(x,y)=y/x$.
Under the assumption that both $\mu_1$ and $\hat{\mu}_1$ have positive lower bound, the function $f$ is twice continuously differentiable on this region.
By the second-order Taylor expansion of a bivariate function with a Lagrange remainder, we have
\[
\begin{aligned}
f(\mu_1,\mu_2)-f(\hat{\mu}_1,\hat{\mu}_2)
&=
\nabla f(\hat{\mu}_1,\hat{\mu}_2)^\top
\begin{pmatrix}
\mu_1-\hat{\mu}_1\\
\mu_2-\hat{\mu}_2
\end{pmatrix}
+ R_\mu \\
=&
-\frac{\hat{\mu}_2}{\hat{\mu}_1^2}(\mu_1-\hat{\mu}_1)
+\frac{1}{\hat{\mu}_1}(\mu_2-\hat{\mu}_2)
+ R_\mu,
\end{aligned}
\]
where the remainder term $R_\mu$ satisfies
\[
|R_\mu|
\le
C\Big(
(\hat{\mu}_1-\mu_1)^2
+
|\hat{\mu}_1-\mu_1|\,|\hat{\mu}_2-\mu_2|
\Big),
\]
for some bounded constant $C$. Therefore, 
\[
\begin{aligned}
R_2(\hat{\mathbb{P}},\mathbb{P})
=&
\int
\frac{1}{\hat{\pi}\hat{\mu}_1}
(\pi-\hat{\pi})(\mu_2-\hat{\mu}_2)
\,d\mathbb{P}(x) \\
&\quad -
\int
\frac{\hat{\mu}_2}{\hat{\pi}\hat{\mu}_1^2}
(\pi-\hat{\pi})(\mu_1-\hat{\mu}_1)
\,d\mathbb{P}(x)
+
\int R_\mu\,d\mathbb{P}(x).
\end{aligned}
\]
\end{proof}

\subsection{Proof of Theorem \ref{theo:asym}}
\begin{proof}
By the decomposition
\[
\hat\psi^{dr}_a-\psi_a(\mathbb{P})
=
(\mathbb{P}_n-\mathbb{P})\{\phi_a(\mathbb{P})\}
+
(\mathbb{P}_n-\mathbb{P})\{\phi_a(\hat{\mathbb{P}})-\phi_a(\mathbb{P})\}
+
R_2(\hat{\mathbb{P}},\mathbb{P}),
\]
the first term $(\mathbb{P}_n-\mathbb{P})\{\phi_a(\mathbb{P})\}$ is $\sqrt{n}$-consistent and asymptotically normal by the central limit theorem. 
It therefore suffices to show that under
\[
\|\hat{\pi}(a|\mathbf{x}) - \pi(a|\mathbf{x})\| = o_{\mathbb{P}}(n^{-1/4}),
\]
\[
\|\hat\mu_j(\mathbf{x},a)-\mu_j(\mathbf{x},a)\| = o_{\mathbb{P}}(n^{-1/4}),\ j=1,2,
\]
and the existence of a constant $\varepsilon>0$ such that $\hat\pi(x,a)\ge \varepsilon$ and $\hat\mu_1(x,a)\ge \varepsilon$ (with probability $1$),
we have $R_2(\hat{\mathbb{P}},\mathbb{P})=o_{\mathbb{P}}(n^{-1/2})$.

By Lemma~\ref{lem:vonmises} and an application of the Cauchy--Schwarz inequality, we have
\[
\begin{aligned}
|R_2(\hat{\mathbb{P}},\mathbb{P})|
\le\;&
C_1\|\hat\pi-\pi\|
\big(\|\hat\mu_1-\mu_1\|+\|\hat\mu_2-\mu_2\|\big) \\
&+
C_2\Big(
\|\hat\mu_1-\mu_1\|^2
+
\|\hat\mu_1-\mu_1\|\,\|\hat\mu_2-\mu_2\|
\Big).
\end{aligned}
\]
where $C_1,C_2<\infty$.
Hence, under the above $o_{\mathbb{P}}(n^{-1/4})$ convergence conditions,
\[
R_2(\hat{\mathbb{P}},\mathbb{P})=o_{\mathbb{P}}(n^{-1/2}).
\]
Together with Assumption~1, which states that $(\mathbb{P}_n-\mathbb{P})\{\phi_a(\hat{\mathbb{P}})-\phi_a(\mathbb{P})\}=o_{\mathbb{P}}(n^{-1/2})$, Slutsky’s theorem implies
\[
\hat\psi^{dr}_a-\psi_a(\mathbb{P})=o_{\mathbb{P}}(n^{-1/2}),
\]
and asymptotic normality follows.
\end{proof}

\subsection{Proof of Theorem \ref{thm:targ}}
\begin{proof}
Define
\[
S = \frac{(Y_2 - \hat{\mu}_2)}{\hat{\mu}_1}
- \frac{(Y_1 - \hat{\mu}_1)\hat{\mu}_2}{\hat{\mu}_1^2},
\quad 
\check{\epsilon}_n
=
\mathbb{P}\!\left[\frac{S}{\hat{\pi}_n} \,\middle|\, A = a \right]
\Big/
\mathbb{P}\!\left[\hat{\pi}_n^{-2} \,\middle|\, A = a\right].
\]

From the proof of Lemma~3 in \cite{nie2021vcnet}, which relies only on properties of spline functions and is independent of the target estimand, we have
\[
\|\hat{\epsilon}_n - \check{\epsilon}_n\|
= O_p\!\big(n^{-1/3}\sqrt{\log n}\big).
\]
Let
$
\hat{\psi}_a^{\mathrm{tr}}
=
\frac{1}{n} \sum_{i=1}^n
\left(
\frac{\hat{\mu}_2}{\hat{\mu}_1}
+ \frac{\hat{\epsilon}(a_i)}{\hat{\pi}}
\right).
$
Then,
\[
\begin{aligned}
\hat{\psi}_a^{\mathrm{tr}} - \psi_a
&=
\frac{1}{n} \sum_{i=1}^n
\left(
\frac{\hat{\mu}_2}{\hat{\mu}_1}
+ \frac{\hat{\varepsilon}(a)}{\hat{\pi}}
- \psi_a
\right) \\
&\le
\left\|
\hat{\varepsilon}(a)
\int_{\mathcal{X}} \frac{1}{\hat{\pi}} \, d\mathbb{P}_n(x)
-
\mathbb{P}\!\left(
\frac{\delta(A=a) S}{\hat{\pi}}
\right)
\right\| \\
&\quad
+ \left\|
\mathbb{P}\!\left[
\frac{\delta(A=a) S}{\hat{\pi}}
+ \frac{\hat{\mu}_2}{\hat{\mu}_1}\right]
- \psi_a
\right\| \\
&\quad
+ \left\|
\frac{1}{n} \sum_{i=1}^n \frac{\hat{\mu}_2}{\hat{\mu}_1}
-
\mathbb{P}\!\left( \frac{\hat{\mu}_2}{\hat{\mu}_1} \right)
\right\| \\
&=: T_1 + T_2 + T_3 .
\end{aligned}
\]

For the first term,
\[
\begin{aligned}
T_1
&=
\left\|
\hat{\varepsilon}(a)
\int_{\mathcal{X}} \frac{1}{\hat{\pi}} \, d\mathbb{P}_n(x)
-
\pi(a)\,
\mathbb{P}\!\left( \frac{S}{\hat{\pi}} \,\middle|\, A = a \right)
\right\| \\
&\le
\left\|
\bigl(\hat{\varepsilon}(a) - \tilde{\varepsilon}(a)\bigr)
\int_{\mathcal{X}} \frac{1}{\hat{\pi}} \, d\mathbb{P}_n(x)
\right\| \\
&\quad +
\left\|
\tilde{\varepsilon}(a)\,
\mathbb{P}\!\left[
\int_{\mathcal{X}} \frac{1}{\hat{\pi}} \, d\mathbb{P}_n(x)
-
\pi(a)\,
\mathbb{P}\!\left( \frac{1}{\hat{\pi}^2} \,\middle|\, A = a \right)
\right]
\right\| \\
&\lesssim
\|\hat{\varepsilon} - \tilde{\varepsilon}\| +
\left\|
\tilde{\varepsilon}(a)
\int_{\mathcal{X}} \frac{1}{\hat{\pi}} \, d(\mathbb{P}_n - \mathbb{P})(x)
\right\| \\
&\quad
+
\left\|
\mathbb{P}\!\left(
\frac{\mu_2 - \hat{\mu}_2}{\hat{\pi}\hat{\mu}_1}
\,\middle|\, A = a
\right)
\int_{\mathcal{X}}
\frac{\hat{\pi} - \pi}{\hat{\pi}} \cdot \frac{1}{\hat{\pi}}
\, d\mathbb{P}(x)
\right\| \\
&\quad
+
\left\|
\mathbb{P}\!\left(
\frac{(\mu_1 - \hat{\mu}_1)\hat{\mu}_2}{\hat{\pi}\hat{\mu}_1}
\,\middle|\, A = a
\right)
\int_{\mathcal{X}}
\frac{\hat{\pi} - \pi}{\hat{\pi}} \cdot \frac{1}{\hat{\pi}}
\, d\mathbb{P}(x)
\right\| \\
&=
O_{\mathbb{P}}\!\left(
n^{-1/3}\sqrt{\log n}
+ r_1(n) r_2(n)
+ r_1(n) r_3(n)
\right).
\end{aligned}
\]

For the second term,
\begin{align}
T_2
&=
\left\|
\mathbb{P}\!\left[
\frac{\delta(A=a) S}{\hat{\pi}}
+ \frac{\hat{\mu}_2}{\hat{\mu}_1}\right]
- \psi_a
\right\| \\
&=
\Big\|
\int \phi_a(\hat{\mathbb{P}})
+ \hat{\psi}_a
- \psi_a
\Big\|
=
\|R_2(\hat{\mathbb{P}}, \mathbb{P})\|.
\end{align}

From Lemma~\ref{lem:vonmises} and the proof of Theorem~\ref{theo:asym}, we obtain
\[
T_2
=
O_p\!\big(
r_1(n)r_2(n)
+ r_1(n)r_3(n)
+ r_2(n)r_3(n)
+ r_2(n)^2
\big).
\]
Under Assumption~(iii),
\[
T_3 = O_p(n^{-1/2}).
\]
Therefore,
\[
\|\hat{\psi}_a^{\mathrm{tr}} - \hat{\psi}_a^{\mathrm{dr}}\|
=
O_p\!\big(
n^{-1/3}\sqrt{\log n}
+ r_1(n)r_2(n)
+ r_1(n)r_3(n)
+ r_2(n)r_3(n)
+ r_2(n)^2
\big).
\]
\end{proof}

\subsection{Synthetic and Semi-synthetic Data Generation}
\textbf{Synthetic data generation.} We first generate covariates $\mathbf{X} \sim \text{Unif}(0,1) \in \mathbb{R}^8$, 
and the treatments and  outcomes are generated as follows (the basic forms of all components are kept consistent with related work \cite{nie2021vcnet}):

\[
\begin{aligned}
\tilde{a} \mid x = {} &
\frac{
10 \sin(\max(x_1, x_2, x_3)) + \max(x_3, x_4, x_5, x_8)^3
}{
1 + (x_1 + x_5)^2
} \\
&+ \sin(0.5x_3)\bigl(1 + \exp(x_4 + x_8 - 0.5x_3)\bigr)\\
& + x_3^2 + 2\sin(x_4) + \cos(x_8)
+ 2x_5 - 6.5
+ \mathcal{N}(0, 0.25), \\
a = & \sigma(\tilde{a}).
\end{aligned}
\]
where $\sigma(\cdot) = (1 + \exp(-(\cdot)))^{-1}$ denotes the sigmoid function.

\[
\begin{aligned}
Z_1 \mid x,a &= \cos\bigl((t-0.5)\cdot 2\pi\bigr)\Bigg(
    t^{2} + \frac{4\cdot\max(x_1,x_6)^{3}}{1 + 2 x_3^{2}}\sin(x_4)
\Bigg),\\[6pt]
Z_2 \mid x,a &= \sin\bigl((-0.5\,t + 0.1)\cdot 2\pi\bigr)(
    -t^{2} + t + \frac{4\cdot\max(x_1,x_7)^{3}}{0.5 + x_3^{3} - x_3}\sin(x_8)
), \\
Y_1 & =  \mathcal{B}(1, 0.7 \sigma(Z_1)), \quad
Y_2  =  Y_1 \times \mathcal{B}(1, 0.9 \sigma(Z_2)).
\end{aligned}
\]

Here, $\mathcal{B}(1,p)$ denotes a Bernoulli random variable with mean $p$.
$Y_1$ and $Y_2$ are binary random variables corresponding to whether a click occurs and whether a conversion occurs, respectively. By construction, $Y_2 = 0$ if $Y_1 = 0$.

\textbf{Semi-synthetic data generation.}We first generate a set of parameters $\boldsymbol{V}_{i}=\boldsymbol{U}_{i}/||\boldsymbol{U}_{i}||$ and $i=1,2,3,4$, where $\boldsymbol{U}_{i}$ is sampled from a normal distribution $\mathcal{N}(\mathbf{0}, \mathbf{1})$,
then:

\[
a \sim \mathrm{Beta}\!\left( 2, \ \left| \frac{v_4^\top \mathbf{x}}{2\, v_3^\top \mathbf{x}} \right| \right).
\]

\[
\begin{aligned}
Z_1 \mid \mathbf{x}, t
&= 2 \Big(4 (t-0.5)^2 \, \sin\big(0.5 \pi t\big)\Big) \\ &\times
\Bigg(
\max\Big(-2, \min\Big(2, \exp\big(0.3 (\pi \tfrac{v_3^\top \mathbf{x}}{v_4^\top \mathbf{x}} - 1)\big)\Big)\Big)
+ 20\, v_1^\top \mathbf{x}
\Bigg), \\[1em]
Z_2 \mid \mathbf{x}, t
&= 2 \Big(-2 (t+0.1)^2 \, \cos\big(0.5 \pi t^2\big)\Big) \\ &\times
\Bigg(
\max\Big(-1, \min\Big(1, \exp\big(0.5 (\pi \tfrac{v_3^\top \mathbf{x}}{v_4^\top \mathbf{x}} - 1)\big)\Big)\Big)
+ 20\, v_2^\top \mathbf{x}
\Bigg), \\
Y_1
&= \mathcal{B}\big(1, 0.6\, \sigma(Z_1)\big), \quad
Y_2
= Y_1 \times \mathcal{B}\big(1, 0.6\, \sigma(Z_2)\big).
\end{aligned}
\]

\end{document}